\documentclass[twocolumn]{autart}   

\usepackage{graphics} 
\usepackage{epsfig} 
\usepackage{amsmath} 
\usepackage{amssymb} 
\usepackage[normalem]{ulem}
\usepackage{url}
\usepackage{amsfonts}
\usepackage{mathtools}
\usepackage{comment}  
\usepackage{dsfont}
\usepackage{cite}
\usepackage{algorithm}
\usepackage{algpseudocode}
\usepackage{booktabs}
\usepackage{siunitx}
\usepackage[table]{xcolor}
\usepackage{stmaryrd}
\usepackage{graphicx}         
\newif\ifshowproofs
\showproofsfalse   
\showproofstrue  

\ifshowproofs
\else
  \excludecomment{pf}
\fi

\newcommand{\shadelightred}{\cellcolor{red!8}}
\newcommand{\shademedred}{\cellcolor{red!16}}
\newcommand{\shadered}{\cellcolor{red!28}}

\newcommand{\shadelightblue}{\cellcolor{blue!8}}
\newcommand{\shademedblue}{\cellcolor{blue!16}}
\newcommand{\shadeblue}{\cellcolor{blue!24}}
\makeatletter
\@ifundefined{labelindent}{}{}
\makeatother
\usepackage{enumitem}

\usepackage{tikz}
\usepackage{circuitikz}
\usepackage{caption}
\usepackage{subcaption}
\usetikzlibrary{positioning}
\usetikzlibrary{arrows}
\usepackage{pgfplots}
\pgfplotsset{compat=1.9}
\usetikzlibrary{datavisualization}
\usetikzlibrary{datavisualization.formats.functions}
\usetikzlibrary{circuits.ee.IEC,decorations.markings}
\usepgfplotslibrary{polar}
\usetikzlibrary{positioning,calc}
\usetikzlibrary{arrows}
\usepackage{pgfplots}
\pgfplotsset{compat=1.17}
\usetikzlibrary{datavisualization}
\usetikzlibrary{datavisualization.formats.functions}
\usepgfplotslibrary{fillbetween}
\usetikzlibrary{arrows.meta}
\usetikzlibrary{arrows,positioning,decorations.pathreplacing}
\usetikzlibrary{fit}

\begin{document}

\begin{frontmatter}

\title{Convex training of Lipschitz-regularized shallow neural networks\thanksref{footnoteinfo}} 

\thanks[footnoteinfo]{This work was supported by National Science and Engineering Research Council of Canada (NSERC) under grants RGPIN-2023-04235 and ALLRP-579869-22.}

\author[Paestum]{Chao Yin}\ead{chao.yin@polymtl.ca} \ and   
\author[Paestum]{Antoine Lesage-Landry}\ead{antoine.lesage-landry@polymtl.ca}               

\address[Paestum]{Polytechnique Montr\'eal, GERAD \& Mila, Montr\'eal, QC, Canada, H3T 1J4.}  


\begin{abstract}
In this work, we introduce a training procedure for shallow neural networks that promotes robustness against adversarial attacks. We solve a non-convex Lipschitz-regularized training program by introducing a convex restriction that can be efficiently solved to global optimality. Our approach can be employed as a post-processing step by taking a pre-trained network as an initial solution to then solving the convex program whose optimal network is guaranteed to be no worse than the initial one. We illustrate the improvements of our training procedure with experiments using real world datasets for regression tasks under an adversarial setting. We show numerically that solving our proposed convex program yields networks with lower objective values on the Lipschitz-regularized program compared to existing methods. Additionally, we show that on certain datasets, networks obtained using our convex training program are both more accurate and robust with respect to adversarial attacks.\end{abstract}
\end{frontmatter}
\section{Introduction}
Neural networks (NNs) are versatile modelling tools due to their ability to approximate a wide class of non-linear functions.
{They have notably found application in various dynamical systems and control tasks such as automated driving \cite{control_reference_1} and demand response control \cite{control_reference_2}. }
Typically, training NNs is done using stochastic gradient descent (SGD)~\cite{sgd} which does not guarantee convergence to a global optimum and is sensitive to hyperparameter choice and requires heavy tuning~\cite{hyperparameter-sensitivity}.

To tackle these issues, shallow neural networks (SNNs), i.e., NNs consisting of a single hidden layer, are considered and~\cite{SCNN_Pilanci} proposes a polynomial-time convex training program.
Their method requires sampling ReLU activation patterns on the data set and imposing a cone constraint for each sampled pattern. The authors of \cite{mishkin2025fastconvexoptimizationtwolayer} propose a more scalable approach by solving  \cite{SCNN_Pilanci}'s convex program with relaxed constraints to leverage GPU-enabled first-order optimization methods. The quality of the solutions obtained from solving \cite{SCNN_Pilanci}'s convex programs are sensitive to the choice of the activation patterns and can perform much worse than solutions obtained with SGD. To mitigate this, activation patterns obtained by first training a SNN with SGD can be employed during resolution. However, \cite{mishkin2025fastconvexoptimizationtwolayer,SCNN_Pilanci}'s formulation yields an SNN with twice as many hidden-units as the number activation patterns. One can either sample half of the remaining activation patterns, which still offers no guarantee on the quality of the solution to the convex program, or, sample all patterns which guarantees that the initial SNN is a feasible solution to the convex program but would in turn use twice as many hidden-units.

In addition to training, another challenge for NNs is their vulnerability to \textit{adversarial attacks}~\cite{fgsm,szegedy2014intriguingpropertiesneuralnetworks} which are small perturbations to the input data to alter the prediction of the trained model during deployment. The type of adversary or source for these perturbations can either be random noises encountered in practical settings, e.g., measurement or telemetry errors, or worst-case perturbations. These perturbations can modify the predicted class for classification, or deviate the predicted value from the target for regression tasks. A model is said to be \textit{adversarially robust} if it is resistant to adversarial attacks, in the sense that a higher effort is needed to modify the data input in order to alter the model prediction. 

Adversarial robustness is essential for ML models to be deployed in safety critical applications, such as autonomous driving~\cite{auto-driving} or security assessment of power systems~\cite{power-systems-verification}. A model's adversarial robustness can be evaluated via {certification} methods \cite{SoK} which provides different types of guarantees of adversarial robustness. The Lipschitz constant of an NN is an example of a certificate that can be interpreted as a sufficient condition for the absence of an adversarial example. A large Lipschitz constant indicates that the NN can be sensitive to small perturbations on certain inputs. Computing the Lipschitz constant for NNs is, however, NP-hard \cite{NP_hard}. To circumvent this issue, \cite{fazylab} efficiently estimates the Lipschitz constant of NNs by solving a semi-definite program (SDP). Several improvements on their work have since been proposed to make their SDP-based method more scalable \cite{automatica_chordal, automatica_CNN_lipschitz, automatica_pose}. Reference \cite{Lipschitz_Pauli} later shows that \cite{fazylab}'s SDP can be used to formulate a Lipschitz-regularized NN training program, then solved using the alternating direction method of multipliers (ADMM)\cite{admm}. Their approach differs from the standard training program in that it promotes adversarial robustness of the resulting model by minimizing the robustness certificate, which in this case is the Lipschitz constant, instead of the standard $\ell_2$- or $\ell_1$- penalty. However, their ADMM-based algorithm suffers from lack of convergence guarantees due to the training problem being non-convex. Moreover, their approach introduces extra hyperparameters beyond those used for SGD as part of the ADMM-based resolution. Alternatively, adversarial training~\cite{fgsm,PGD} is another adversarially robust training approach. It involves solving min-max optimization problems which minimizes the \textit{adversarial loss} induced by the worst-case perturbation to the training input. During training, worst-case perturbations are approximated using projected gradient descent attacks (PGD-attacks) \cite{PGD} and the NN weights are then updated to minimize the loss on the perturbed input. PGD-attacks are also a popular tool to evaluate the adversarial robustness of NNs post training. However, \cite{PGD} only solves the min-max problem approximately and without convergence guarantees. To address this issue, \cite{bai} proposes a convex surrogate program inspired by \cite{SCNN_Pilanci} for the min-max problem of SNNs which can be efficiently solved to global optimality. 
Similarly motivated by the aforementioned challenges of the ADMM approach to solve the Lipschitz-regularized training problem \cite{Lipschitz_Pauli}, our work provides a tractable convex program to train an adversarially robust SNN by minimizing the SNN's sensitivity to norm-bounded perturbations, namely, its Lipschitz constant. Our convex program can be viewed as a convex restriction which can be constructed from any solution and obtain the optimal SNN with the same activation patterns. This, in turn, addresses the challenge of finding good quality activation patterns as solving our convex restriction is guaranteed to yield a solution that is no worse than any initial solution. In contrast, the convex programs of existing approaches \cite{bai, mishkin2025fastconvexoptimizationtwolayer,SCNN_Pilanci} offer no guarantee that solving their program is advantageous compared to SGD. 
Our proposed program can be solved using off-the-shelf interior-point-based solvers like MOSEK~\cite{mosek-pythonapi}. 

The main contributions of this work are as follows:
\begin{enumerate}

{
\item We develop a convex training method for Lipschitz-regularized SNNs with ReLU activations. We show that fixing the outer layer weights and hidden-units activation patterns yields a convex restriction of the original non-convex training problem. Building on this result, we propose an iterative algorithm which solves a sequence of these convex restrictions to global optimality and prove that the objective value decreases monotonically. 
\item We propose techniques for constructing stronger convex restrictions. We show that a convex restriction can be initialized by extracting the activation patterns and outer layer weights of any pre-trained SNN, and demonstrate that solving this convex restriction to global optimality is guaranteed to yield an SNN with an objective value that is no worse than the pre-trained SNN. 
\item We experimentally illustrate on UCI and synthetic datasets that using our post-processing method for Lipschitz regularization yields more robust SNNs with respect to PGD \cite{PGD} adversarial attacks. We show on both $\ell_2$- and Lipschitz-regularized training problems, our post-processing method guarantees improvement on the objective value. 
}
\end{enumerate}
\section{Background}
In this section, we provide the background concepts needed to introduce our method. We then present intermediary results and analysis in preparation to our main contributions. 
\subsection{Lipschitz constant estimation for NNs}
Consider an SNN of $m\in\mathbb{N}$ hidden units defined by the function $f:\mathbb{R}^d\mapsto\mathbb{R}$ with ReLU activation functions and a scalar output. Let the collection of vectors $\{\mathbf{u}_j\}_{j=1}^m\in\mathbb{R}^d$ be the hidden units weights, and the collection of scalars $\{\alpha_j\}_{j=1}^m\in\mathbb{R}$ be the outer layer weights. The SNN is given by
\begin{equation}
\label{eq:NN_function}
    f(\mathbf{x}) = \sum_{j=1}^m\alpha_j(\mathbf{x}^\top\mathbf{u}_j)_+,
\end{equation} where $(\cdot)_+=\max(0,\cdot)$ is the ReLU activation function and $\mathbf{x}\in\mathbb{R}^d$ is the input feature vector. A bias term can be incorporated by appending a feature of $1$ to the input feature vector. Hereinafter, we assume that such a value is appended to $\mathbf{x}$ and that for a hidden unit $\mathbf{u}_j = \begin{bmatrix}u_1,...,u_{d+1}\end{bmatrix}^\top$, where $u_{d+1}$ is the bias term. 
We write $\mathbf{\hat u}_j=\begin{bmatrix}u_1,...,u_{d}\end{bmatrix}^\top$ when omitting the bias term. Let $\mathbf{r}\in\mathbb{R}_{>0}^m$ be a rescaling factor. We observe that the weights of an SNN can be rescaled with $\mathbf{R}:=\text{diag}(\mathbf{r})$ to obtain $\mathbf{\alpha'}=\mathbf{R}\mathbf{\alpha}$ and $\mathbf{u'}=\mathbf{R}^{-1}\mathbf{u}$ without affecting its output: 
\[
\sum_{j=1}^m\alpha_j(\mathbf{x}^\top\mathbf{u}_j)_+ =  \sum_{j=1}^mr_j\alpha_j(\frac{1}{r_j}\mathbf{x}^\top\mathbf{u}_j)_+= \sum_{j=1}^m\alpha_j'(\mathbf{x}^\top\mathbf{u}_j')_+.
\] This property will be in the next section's analysis.

To certify a model's adversarial robustness, we use a {Lipschitz upper bound} $L \in \mathbb{R}_{\geq0}$ of the SNN, defined as
\[
|f(\mathbf{x})-f(\mathbf{y})|\leq L{\|\mathbf{x}-\mathbf{y}\|_2}, \quad \forall \mathbf{x},\mathbf{y}\in\mathbb{R}^d.
\]
The lowest or tightest Lipschitz upper bound $L$ is referred to as the Lipschitz constant. In the context of adversarial robustness, the Lipschitz upper bound serves as certificate on the robustness of the regression model. We can interpret $L$ as an upper bound on the variation of the model output resulting from a unit variation in the input. Because computing the Lipschitz constant of a neural network is NP-hard \cite{NP_hard}, it is often preferable to efficiently estimate a tight upper bound. For NNs of arbitrary depth, \cite{fazylab} propose a tight Lipschitz constant estimation method based on semi-definite programming. For SNNs in particular, the constant estimation boils down to solving (\ref{eq:fazlyab_LMI}) given below. Specifically, \cite{fazylab} shows that if a $\rho \in \mathbb{R}_{\geq0}$ satisfies the linear matrix inequality (LMI)~(\ref{constraint:LMI}) for some auxiliary variable $\boldsymbol{\lambda}\in\mathbb{R}_{\geq0}^m$, then $\sqrt{\rho}$ is a Lipschitz upper bound for the SNN with weights $\{\mathbf{u}_j\}_{j=1}^m$ and $\{\alpha_j\}_{j=1}^m$. The objective is then to minimize $\rho$ to identify the tightest Lipschitz upper bound satisfying the LMI~\eqref{constraint:LMI}. Let $\mathbf{T} = \mathrm{diag}(\boldsymbol{\lambda})$, and $
\mathbf{U}^\top = \begin{bmatrix}\mathbf{\hat u}_1,...,\mathbf{\hat u}_m\end{bmatrix}$, consider the following Lipschitz upper bound estimation program for SNNs: 
\begin{subequations}
\begin{align}
\min_{\substack{\rho\in\mathbb{R}_{\geq0},\  \boldsymbol \lambda \in \mathbb{R}_{\geq0}^m}}  & \  \rho\\
\text{s.t.} \quad\;\;\; & {H}(\mathbf{T},\rho,\{\mathbf{u}_j\}_{j=1}^m,\{\alpha\}_{j=1}^m)
\preceq 0 \label{constraint:LMI},
\end{align}
\label{eq:fazlyab_LMI}
\end{subequations}
where,
\begin{equation}
\label{matrix H}
{H}(\mathbf{T},\rho,\{\mathbf{u}_j\}_{j=1}^m,\{\alpha_j\}_{j=1}^m)=
\begin{bmatrix}
-\rho \mathbf{I}_d & \mathbf{U}^\top \mathbf{T} & 0\\
\mathbf{T}\mathbf{U} & -2\mathbf{T} & \boldsymbol{\alpha}\\
0 & \boldsymbol{\alpha}^\top & -1
\end{bmatrix}.
\end{equation}
Problem (\ref{eq:fazlyab_LMI}) can be used post training to certify the adversarial robustness of a SNN where the model weights $\{\mathbf{u}_j\}_{j=1}^m\text{ and }\{\alpha_j\}_{j=1}^m$ are fixed and $\rho\in \mathbb{R}_{\geq0},\ \boldsymbol{\lambda}\in \mathbb{R}_{\geq0}^m$ are variables. The authors of \cite{fazylab} note that satisfying the LMI~(\ref{constraint:LMI}) is only a sufficient condition for $\rho$ to be a Lipschitz upper bound, therefore solving \eqref{eq:fazlyab_LMI} does not generally yield the Lipschitz constant. Reference \cite{Lipschitz_Pauli} proposes to use \cite{fazylab}'s method to regularize a NN’s Lipschitz
constant during training.
For SNNs in particular, this is done by allowing SNN weights $\{\mathbf{u}_j\}_{j=1}^m\text{ and }\{\alpha_j\}_{j=1}^m$ to be variables and adding the prediction loss to \eqref{eq:fazlyab_LMI}.
Given a data matrix $\mathbf{X}\in\mathbb{R}^{n\times d+1}$, consider the following non-convex training program for a SNN incorporating both $\ell_2$-regularization and Lipschitz-regularization:
\begin{subequations}\label{eq:Lip-NCNN}
\begin{align}
\min_{\substack{\rho\in \mathbb{R}_{\geq 0},\\\{\mathbf{u}_j\}_{j=1}^m\in \mathbb{R}^d\\\{\alpha_j\}_{j=1}^m\in\mathbb{R},\\ \boldsymbol\lambda\in \mathbb{R}_{\geq0}^{m}}}
& \ \frac{1}{2}\Big \| \sum_{j=1}^m \alpha_j(\mathbf{X}\mathbf{u}_j)_+ - \mathbf{y}\Big\|_2^2\label{eq:Lip-NCNN-obj}\\[-20pt]
  & \ \;\; + \frac{1}{2}\beta_1\sum_{j=1}^m\left(\|\mathbf{u}_j\|_2^2 + |\alpha_j|^2\right)
      + \beta_2\,\rho \nonumber\\
\text{s.t.}\quad \;\;
& H\big(\mathbf{T},\{\mathbf{u}_j\}_{j=1}^m,\{\alpha_j\}_{j=1}^m,\rho\big)\preceq 0, \label{eq:Lip-NCNN-con} \end{align} \end{subequations} 
where $\beta_1\geq0$ and $\beta_2\geq0$ are, respectively, the $\ell_2$- and Lipschitz-regularizer scaling factors.
{
For the remainder of this paper, we refer to \eqref{eq:Lip-NCNN} as one of two cases: (i) $\beta_1>0$ and $\beta_2=0$ for the $\ell_2$-regularized training problem; (ii) $\beta_1=0$ and $\beta_2>0$ for the Lipschitz-regularized problem.
}

Note that when we refer to the Lipschitz-regularized~\eqref{eq:Lip-NCNN}, it is assumed that the variables $\mathbf{T}$ are fixed to some feasible values unless specified otherwise. When matrix~\eqref{matrix H} is defined with a fixed $\mathbf{T}$, we write $H_{\mathbf{T}}$ to emphasize that $\mathbf{T}$ is considered as a parameter. To avoid the bilinearity now present in the LMI~\eqref{constraint:LMI} due to SNN weights being variables, \cite{Lipschitz_Pauli} propose to fix the variables $\boldsymbol{\lambda}
$ in $\mathbf{T}$. They obtain an initial feasible $\mathbf{T}$ by first training the NN with $\ell_2$-regularization using SGD, then, solving \eqref{eq:fazlyab_LMI} using the weights of the trained NN to obtain a feasible initial $\mathbf{T}$ to begin solving the Lipschitz-regularized~\eqref{eq:Lip-NCNN}.  
ADMM is employed to exploit the fact that (\ref{eq:Lip-NCNN-obj}) can be split into a non-convex and a convex part. 
We refer the reader to \cite{Lipschitz_Pauli} for detailed implementation as we later use their ADMM procedure without modification as a baseline in Section~\ref{sec:experiments}. It is worth highlighting that the work of~\cite{Lipschitz_Pauli} can be applied to NNs of any number of hidden layers as well as vector outputs.
In \cite{Lipschitz_Pauli}, the prediction loss, e.g., mean squared error (MSE) for regressions, remains non-convex in the variables $\{\mathbf{u}_j\}_{j=1}^m$ and $\{\alpha_j\}_{j=1}^m$. Consequently, the ADMM algorithm is not guaranteed to converge. In our work, we propose a convex restriction of the Lipschitz-regularized~\eqref{eq:Lip-NCNN} which can be efficiently solved to optimality. Moreover, when given an initial solution to \eqref{eq:Lip-NCNN}, the solution to the convex restriction is guaranteed to be no worse. We later demonstrate experimentally that solving this convex restriction often obtains solutions of significantly higher quality when used on a pre-trained SNN employing the ADMM based method.
\subsection{Convex Reformulation for Training SNNs with ReLU} \label{subsec:my-subsection}\label{sec:pilanci backgorund}
Reference \cite{mishkin2025fastconvexoptimizationtwolayer} shows that training ReLU SNNs with $\ell_2$-regularized~\eqref{eq:Lip-NCNN} can be reformulated as a convex program. We now introduce the tools required to derive this convex formulation, in addition to our restriction of the Lipschitz-regularized training program. We refer to an activation pattern of a hidden unit on data $\mathbf{X}$ as the vector $\mathbf{s} \in \{0,1\}^n$, which indicates the activation of the ReLU function on each data point. Formally, $\mathbf{s}$ is an activation pattern if there exists some $\mathbf{u}\in\mathbb{R}^{d+1}$ that satisfies \begin{equation}\label{eq:basic pattern constraint}
(2\mathbf{D}(\mathbf{s})-\mathbf{I})\,\mathbf{X}\mathbf{u}\ge \mathbf{0},
\end{equation}
where $\mathbf{D}(\mathbf{s}):=\text{diag}(\mathbf{s})$. If $\mathbf{u}$ is an optimization variable, the conic constraint in \eqref{eq:basic pattern constraint} forces $\mathbf{u}$ to have the activation pattern of $\mathbf{s}$. Let $M\in\mathbb{N}$ be the total number of unique activation patterns on the data $\mathbf{X}$. We write $\mathbf{s}^k$ to denote the activation pattern of index $k$ among the~$M$ unique ones. 
We define $\mathcal{K}=\{k_i\}_{i=1}^{\tilde M} $, with each element belonging to  $\{1,...,M\}\equiv\llbracket 1, M \rrbracket$, as a {multiset of activation patterns} which contains $\tilde M$ indices of activation patterns.
The {multiplicity} of an activation pattern is the number of times it appears in $\mathcal{K}$.
For two multisets $\mathcal{K}_i$ and $\mathcal{K}_j$, we write $\mathcal{K}_i \subseteq \mathcal{K}_j$ if every $k$ appears in $\mathcal{K}_j$ with multiplicity at least as large as in $\mathcal{K}_i$.
Unlike in \cite{mishkin2025fastconvexoptimizationtwolayer}, we admit activation patterns with a multiplicity higher than one. 

Let $\mathcal{K^+}$ and $\mathcal{K^-}$ be two pattern multisets of sizes $\tilde M^+$ and $\tilde M^-$, respectively. Consider the convex program for training SNNs: 
\begin{subequations}\label{eq:Lip-convex-NN-pilanci}
\begin{align}
&\min_{\substack{
\{\mathbf{w}_{i}^{+}\}_{i=1}^{\tilde M^+}\in\mathbb{R}^d\\
\{\mathbf{w}_{i}^{-}\}_{i=1}^{\tilde M^-}\in\mathbb{R}^d}}
\begin{aligned}[t]
   &\tfrac12 \Bigl\|
    \sum_{i=1}^{\tilde M^+}\mathbf{D}(\mathbf{s}^{k_i^+})\mathbf{X}\mathbf{w}_i^{+}
   \\[-0.9ex] &\quad-\sum_{i=1}^{\tilde M^-}\mathbf{D}(\mathbf{s}^{k_i^-})\mathbf{X}\mathbf{w}_i^{-}
   -\mathbf{y}\Bigr\|_2^2\\[-0.9ex]
   &\qquad + \beta_1\Bigl( \sum_{i=1}^{\tilde M^+}\|\mathbf{w}_i^{+}\|_2
        + \sum_{i=1}^{\tilde M^-}\|\mathbf{w}_i^{-}\|_2\Bigr)
\end{aligned} 
\label{eq:Lip-convex-NN-pilanci-obj}
\\[-2pt]
\text{s.t.} \quad &
 (2\mathbf{D}(\mathbf{s}^{k_i^+})-\mathbf{I})\mathbf{X}\mathbf{w}_i^+\ge 0,
\quad \forall i \in \llbracket 1,\, \tilde M^+ \rrbracket,
\label{eq:Lip-convex-NN-pilanci-conp}
\\
& (2\mathbf{D}(\mathbf{s}^{k_i^-})-\mathbf{I})\mathbf{X}\mathbf{w}_i^-\ge 0,
\quad \forall i \in \llbracket 1,\, \tilde M^- \rrbracket.
\label{eq:Lip-convex-NN-pilanci-conn}
\end{align}
\end{subequations}
The following result establishes the equivalence between~\eqref{eq:Lip-convex-NN-pilanci} and the $\ell_2$-regularized~\eqref{eq:Lip-NCNN} and extends \cite[ Theorem 2.1]{mishkin2025fastconvexoptimizationtwolayer}.
\begin{cor}
\label{equivalence convex l2}
 Consider the convex program \eqref{eq:Lip-convex-NN-pilanci} constructed using pattern multisets $\mathcal{K}^\pm$. Suppose there exist minimizers $\big(\{\mathbf{u}_j^{\star}\}^{m}_{j=1},\{\alpha^\star_j\}^{m}_{j=1}\big)$ and $\big(\{\mathbf{w}_i^{{\pm}\star}\}_{i=1}^{\tilde{M}^\pm}\big)$ of the $\ell_2$-regularized~\eqref{eq:Lip-NCNN} and \eqref{eq:Lip-convex-NN-pilanci}, respectively, which satisfy the following assumptions:
\begin{assum}
The number of hidden units $m$ is such that \label{thm1a1}
$
m \ge \tilde m = q^+ + q^-,
$
where $q^+$ and $q^-$ are the number of nonzero entries $\mathbf{w}_i^{+\star}$ and $\mathbf{w}_i^{-\star}$, respectively. Suppose that $\mathbf{w}_i^{+\star}$ and $\mathbf{w}_i^{-\star}$
are ordered so that the nonzero entries come first, followed 
the zero entries.
\end{assum}
\begin{assum}
The corresponding pattern multisets $\mathcal{L^+}$ and $\mathcal{L^-}$ of $\{\mathbf{u}_j^\star,\alpha_j^\star\}_{j=1}^m$ for \eqref{eq:Lip-NCNN}, which are the patterns of hidden units with positive outer weights and negative outer weights, respectively, are such that $\mathcal{L}^+\subseteq \mathcal{K}^+$ and $\mathcal{L}^-\subseteq \mathcal{K}^-$. \label{thm1a2}
\end{assum}
Then, the $\ell_2$-regularized~\eqref{eq:Lip-NCNN} and \eqref{eq:Lip-convex-NN-pilanci} have the same optimal objective value, and an optimal solution for the $\ell_2$-regularized~\eqref{eq:Lip-NCNN} can be constructed via 
\begin{subequations}\label{eq:construction}
\begin{align}
(\mathbf{u}_j^{\star}, \alpha_j^{\star})&\leftarrow (\tfrac{
\mathbf{w}_{j}^{{+}\star}}{\sqrt{\|\mathbf{w}_{j}^{{+}\star}\|_2}},\ 
\sqrt{\|\mathbf{w}_{j}^{{+}\star}\|_2}),  j\in\llbracket 1, \tilde{q}^+ \rrbracket \\
(\mathbf{u}_{j+\tilde q^+}^{\star}, \alpha_{j+\tilde q^+}^{\star})&\leftarrow (\tfrac{
\mathbf{w}_{j}^{{-}\star}}{\sqrt{\|\mathbf{w}_{j}^{{-}\star}\|_2}},-\ 
\sqrt{\|\mathbf{w}_{j}^{{-}\star}\|_2}),  j \in \llbracket 1,\, \tilde q^- \rrbracket. 
\end{align}
and $(\mathbf{u}_{j}^{\star}, \alpha_{j}^{\star})$ are assigned zero for $j\in\llbracket \tilde m+1,m \rrbracket$.
\end{subequations} 
\end{cor} Corollary~\ref{equivalence convex l2} shows that under certain conditions, we can solve the convex program \eqref{eq:Lip-convex-NN-pilanci} to optimality and recover an optimal solution for the $\ell_2$-regularized~\eqref{eq:Lip-NCNN} using the mapping~\eqref{eq:construction}. Corollary~\ref{equivalence convex l2} is proven in \cite{mishkin2025fastconvexoptimizationtwolayer} for the special case $\mathcal{K}^+=\mathcal{K}^-$, where $\mathcal{K}^\pm$ contain only unique patterns. To broaden the scope of Corollary~\ref{equivalence convex l2}, we removed the aforementioned requirements as the arguments of \cite{mishkin2025fastconvexoptimizationtwolayer} carries over verbatim \cite[ Theorem 2.1]{mishkin2025fastconvexoptimizationtwolayer}. We now pivot to a discussion on the practical issues of Corollary~\ref{equivalence convex l2} in order to motivate our proposed modifications to the definition of pattern multisets. 
\begin{rem}In the $\ell_2$-regularized case, removing the requirement that $\mathcal{K}^\pm$ must contain unique patterns bears limited impact, as \cite{mishkin2025fastconvexoptimizationtwolayer} have shown that any two hidden units with the same activation pattern and outer weight sign can be summed into a single weight without increasing the objective. This does not necessarily hold true for the Lipschitz-regularized case. We therefore define $\mathcal{K}^\pm$ as multisets which can include patterns with a non-unit multiplicity. \end{rem}

While Corollary~\ref{equivalence convex l2} provides an efficient way to compute an SNN weights, the convex program~\eqref{eq:Lip-convex-NN-pilanci} equivalent to the $\ell_2$-regularized~\eqref{eq:Lip-NCNN} cannot be constructed without the pattern multisets of an optimal solution to the latter problem due to \eqref{thm1a2}. This would require enumerating all activation patterns on $\mathbf{X}$, which is impractical because this number grows exponentially in the rank of the data \cite{SCNN_Pilanci}. Instead of solving~\eqref{eq:Lip-convex-NN-pilanci} with the optimal pattern multisets, \cite{mishkin2025fastconvexoptimizationtwolayer} and \cite{SCNN_Pilanci} propose that patterns can be randomly generated using a Gaussian distribution or extracted from a pre-trained SNN. For instance, \cite{mishkin2025fastconvexoptimizationtwolayer} first solves the non-convex $\ell_2$-regularized training problem $\eqref{eq:Lip-NCNN}$ of $m$ hidden units using SGD. Doing so, they obtain an initial feasible SNN with pattern multisets $\mathcal{L}^\pm$ to construct the convex program \eqref{eq:Lip-convex-NN-pilanci} by taking $\mathcal{K}^+=\mathcal{K}^-=\mathcal{L}^+\cup \mathcal{L}^-$. We show later that this convex program is in fact a convex restriction on the non-convex training problem with $2m$ hidden units. This, in turn, implies the convex restriction is on a different problem than the initial one because of the use of more hidden units. Because the convex restriction contains the initial solution, it is guaranteed that the convex optimal solution yields improvement, albeit for a different problem with twice as many hidden units. 
The first reason this is not desirable is that increasing the number of hidden units might lead to overfitting. The second reason is that it prevents a direct comparison between the SNN obtained from the convex program \eqref{eq:Lip-convex-NN-pilanci} with the initial SNN to determine if improvements are due to the globally optimal solution on the convex restriction and not due to simply increasing the expressivity of the SNN by adding the hidden units. Our proposed definition of pattern multisets allows us to construct a convex restrictions on the original training problem with the same number of hidden units. In the next section, we present our main contribution and propose a convex restriction for the Lipschitz-regularized~\eqref{eq:Lip-NCNN} inspired by \eqref{eq:Lip-convex-NN-pilanci}. 

\section{Lipschitz-regularized convex training}
{
In this section, we present a convex training method for the Lipschitz-regularized~\eqref{eq:Lip-NCNN}. As mentioned previously, our goal is not to solve the Lipschitz-regularized~\eqref{eq:Lip-NCNN} to global optimality. Instead, we aim to improve any given initial feasible solution via solving a convex restriction of Lipschitz-regularized~\eqref{eq:Lip-NCNN} to global optimality. By construction, the initial solution remains feasible for this restricted problem and, therefore, its optimal solution is guaranteed to be no worse than that of the initial solution. This process is formalized in Algorithm~\ref{alg:training_lipschitz}.

We next derive the first step in this convex restriction. Specifically, we fix the non-negative components of the outer weights which yields a non-convex restriction of~\eqref{eq:Lip-NCNN}. We then obtain a convex restriction by fixing the signs of the outer weights as well as the activation patterns of the hidden-units. 
\subsection{Warm-up: SNN with fixed outer layer absolute values}
We re-expresss the outer weights as $\boldsymbol{\alpha}=\text{sign}(\boldsymbol{\alpha})|\boldsymbol{\alpha}|$ element-wise and let $\mathbf{b}=\text{sign}(\boldsymbol{\alpha})$ and $\mathbf{c}=|\boldsymbol{\alpha}|$ be, respectively, the sign and non-negative components of the outer weights, with the convention that  $\operatorname{sign}(x)=1$ for $x\geq 0$ and $-1$ otherwise. Given fixed $\mathbf{c}=\{c_j\}_{j=1}^m>\mathbf{0}$, the following non-convex mixed-integer program trains an SNN with $m$ hidden units using the ReLU activation function:
\begin{subequations}\label{eq:Lip-NCNN-simplified}
\begin{align}
\min_{\substack{
    \rho \in \mathbb{R}_{\geq0},\\
    \{\mathbf{u}_j\}_{j=1}^m \in \mathbb{R}^d,\\
    \{{b}_j\}_{j=1}^m \in \{1,-1\},\\
}}
\;
& \frac{1}{2}\Bigl\|
    \sum_{j=1}^m b_jc_j (\mathbf{X}\mathbf{u}_j)_+ - \mathbf{y}
  \Bigr\|_2^2+ \beta_2\rho
\label{eq:obj_split_SAPN}\\[2pt]
\text{s.t.}\quad\;\;\;\;\;
& {H}_{\mathbf{T}}\!\left(\{\mathbf{u}_j\}_{j=1}^{m},
             \{b_jc_j\}_{j=1}^{m}, \rho\right)\preceq 0.
\label{constraints:LMI_SAPN}
\end{align}
\end{subequations}

We prove that~\eqref{eq:Lip-NCNN-simplified} given $\mathbf{c}$ is a restriction of the Lipschitz-regularized~\eqref{eq:Lip-NCNN} in Lemma~\ref{lemma:NC_restriction} when $\mathbf{T}$ in both problems is fixed to the same value. In the case where $\mathbf{T}$ is free, however, they have the same optimal objective values if an optimal solution exists in \eqref{eq:Lip-NCNN} such that all its outer layer weights are non-zero.
\begin{lem}
\label{lemma:NC_restriction}
Consider the Lipschitz-regularized~\eqref{eq:Lip-NCNN} and~\eqref{eq:Lip-NCNN-simplified} with $\mathbf{c}$ with optimal objective values $p_\eqref{eq:Lip-NCNN}^\star$ and $p^\star_\eqref{eq:Lip-NCNN-simplified}(\mathbf{c})$, respectively. Then, if (i) $\mathbf{T}$ is fixed to the same value in both problems, we have that $p_\eqref{eq:Lip-NCNN}^\star\leq p^\star_\eqref{eq:Lip-NCNN-simplified}(\mathbf{c})$; and (ii) $\mathbf{T}$ is free in both problems and there exists an optimal solution of \eqref{eq:Lip-NCNN} with $\{\alpha_j^\star\}_{j=1}^m$ such that $\alpha_j^\star>0,\ \forall j\in\llbracket 1,m\rrbracket$, we have $p_\eqref{eq:Lip-NCNN}^\star=p^\star_\eqref{eq:Lip-NCNN-simplified}(\mathbf{c})$.
\end{lem}
\begin{pf}
Let $\mathcal{S}_\eqref{eq:Lip-NCNN-simplified}(\mathbf{c})$ be the feasible solution set of \eqref{eq:Lip-NCNN-simplified} using $\mathbf{c}>0$ and $\mathcal{S}_\eqref{eq:Lip-NCNN}$ be that of \eqref{eq:Lip-NCNN} such that for any solution in $\mathcal{S}_\eqref{eq:Lip-NCNN}$, its outer layer weights $\{\alpha_j\}_{j=1}^m$ are such that $\alpha_j>0,\ \forall j\in\llbracket 1,m\rrbracket$.
Consider the surjective mapping $\Phi:\mathcal{S}_\eqref{eq:Lip-NCNN}\mapsto\mathcal{S}_\eqref{eq:Lip-NCNN-simplified}(\mathbf{c})$ defined as 
\begin{subequations}
\label{mapping:4_and_8}
    \begin{align} \notag
        \Phi\big(\{\mathbf{u}_j\}_{j=1}^m,&\{\alpha_j\}_{j=1}^m,\rho,\mathbf{T}\big)=\\ \notag
        &\big(\{|\alpha_j|\mathbf{u}_j/c_j\}_{j=1}^m,\{\alpha_j/|\alpha_j|\}_{j=1}^m,\rho,\mathbf{RTR}\big),
    \end{align}
\end{subequations}
where $\mathbf{R}=\text{diag}(\{c_j/|\alpha_j|\}_{j=1}^m)$. The mapped solution is feasible for \eqref{eq:Lip-NCNN-simplified} with $\mathbf{c}$, because the LMIs in both problems are identical. Additionally, because rescaling does not change the SNN output, the objective value of the mapped solution remains the same. 
To show that $\Phi$ is surjective, consider any solution $\mathbf{s}_\eqref{eq:Lip-NCNN-simplified}=(\{\mathbf{u}_j\}_{j=1}^m,\{b_j\}_{j=1}^m,\rho,\mathbf{T}_2\big)\in \mathcal{S}_\eqref{eq:Lip-NCNN-simplified}(\mathbf{c})$. Consider$$\mathbf{s}_\eqref{eq:Lip-NCNN}=\big(\{\mathbf{u}_j\}_{j=1}^m,\{b_jc_j\}_{j=1}^m,\rho,\mathbf{T}_2\big)\in\mathcal{S}_\eqref{eq:Lip-NCNN}(\mathbf{c}),$$
where the feasible subset $\mathcal{S}_\eqref{eq:Lip-NCNN}(\mathbf{c})\subseteq\mathcal{S}_\eqref{eq:Lip-NCNN}$ is defined as 
\begin{subequations}
    \begin{align}
        \notag
        \mathcal{S}_\eqref{eq:Lip-NCNN}(\mathbf{c}) = \{\mathbf{s}\in \mathcal{S}_\eqref{eq:Lip-NCNN}
       \ |\ c_j=|\alpha_j|,\ \forall j \in \llbracket 1,m\rrbracket \}.
    \end{align}
\end{subequations}
Then, $\mathbf{R=I}$ and $\Phi(\mathbf{s}_\eqref{eq:Lip-NCNN})=\mathbf{s}_{\eqref{eq:Lip-NCNN-simplified}}$. Hence, $p_\eqref{eq:Lip-NCNN}^\star= p^\star_\eqref{eq:Lip-NCNN-simplified}(\mathbf{c})$ if there is at least one optimal solution for \eqref{eq:Lip-NCNN} in $\mathcal{S}_\eqref{eq:Lip-NCNN}$.
In the case where $\mathbf{T}$ is fixed to the same value in both problems, we consider the mapping $\Phi$ when its domain is restricted to the feasible subset $\mathcal{S}_\eqref{eq:Lip-NCNN}(\mathbf{c})\subseteq\mathcal{S}_\eqref{eq:Lip-NCNN}$ and refer to it as $\tilde \Phi$. The mapping $\tilde \Phi$ inherits $\Phi$'s surjectivity. It is also injective because for $\mathbf{s}, \mathbf{s}'\in \mathcal{S}_\eqref{eq:Lip-NCNN}(\mathbf{c})$, we have that $\mathbf{s}\neq \mathbf{s}'\implies \mathbf{u}_j \neq \mathbf{u}_j'$ or $\alpha_j \neq \alpha_j'$  for some $j \in \llbracket 1,m\rrbracket$ and hence $\tilde \Phi(\mathbf{s})\neq \tilde \Phi(\mathbf{s}')$. It follows that, $\tilde \Phi$ is a bijection. 
Because the domain of $\tilde \Phi$ is $\mathcal{S}_\eqref{eq:Lip-NCNN}(\mathbf{c})\subseteq\mathcal{S}_\eqref{eq:Lip-NCNN}$, we have that $p_\eqref{eq:Lip-NCNN}^\star\leq p^\star_\eqref{eq:Lip-NCNN-simplified}(\mathbf{c})$, which holds with equality if and only if there is at least one optimal solution in $\mathcal{S}_\eqref{eq:Lip-NCNN}(\mathbf{c})$. \qed \end{pf}
Letting $\mathbf{T}$ be free introduces bilinearity in the matrix~\eqref{constraints:LMI_SAPN} and, therefore, must be fixed to obtain an LMI. After fixing $\mathbf{T}$, however, \eqref{eq:Lip-NCNN-simplified} becomes a restriction of \eqref{eq:Lip-NCNN} and $p_\eqref{eq:Lip-NCNN}^\star\leq p^\star_\eqref{eq:Lip-NCNN-simplified}$. Lemma~\ref{lemma:bounded_alpha} shows that some choices of $\mathbf{c}$ yields a strict inequality. Specifically, if $\mathcal{S}_\eqref{eq:Lip-NCNN-simplified}(\mathbf{c})$ is non-empty, then $\|\mathbf{c}\|\in[0, \overline{r}]$ where $\overline{r}>0$. Consequently, we have that~\eqref{eq:Lip-NCNN-simplified} can be made infeasible if $\mathbf{c}$ is not chosen appropriately. 

\begin{lem}
\label{lemma:bounded_alpha}
Suppose $\mathbf{T}$ is fixed. Let $\mathcal{S}_\eqref{eq:Lip-NCNN-simplified}(\mathbf{c})$ be the feasible set for~\eqref{eq:Lip-NCNN-simplified} with $\mathbf{c}$. 
There exists $\overline{r}\in\mathbb{R}_{>0}$ such that $\mathcal{S}_{\eqref{eq:Lip-NCNN-simplified}}(\mathbf{c})$ is empty if $\|\mathbf{c}\|> \overline{r}$.
\end{lem}
\begin{pf}
We apply Schur's complement twice on $H_\mathbf{T}(\mathbf{U}, \boldsymbol{\alpha},\rho)\preceq0$ to obtain the following conditions: $\rho\geq0$, which is satisfied by construction, and
\begin{equation}
\label{eq:twice-schur}
    \mathbf{M}=2\mathbf{T}- \boldsymbol{\alpha} \boldsymbol{\alpha}^\top -\rho^{-1}\mathbf{TUU^\top T} \succeq 0,
\end{equation}
where $\boldsymbol{\alpha}=\{b_jc_j\}_{j=1}^m$. We equivalently have that 
$$ 
\boldsymbol{\alpha}^\top \mathbf{M} \boldsymbol{\alpha} \leq 2\mathbf{ \boldsymbol{\alpha}}^\top\mathbf{T} \boldsymbol{\alpha}-\| \boldsymbol{\alpha}\|^4\leq 2\max(\mathbf{T})\|\mathbf{c}\|^2-\| \mathbf{c}\|^4,
$$
where the rightmost term becomes negative when $\|\mathbf{c}\|$ is sufficiently large, leading to $\mathbf{M} \nsucceq0$. \qed
\end{pf}
The consequence of Lemma \ref{lemma:bounded_alpha} is that $\mathbf{c}$ must be carefully chosen because it affects the feasibility of \eqref{eq:Lip-NCNN-simplified}. To guarantee that our choice of $\mathbf{c}$ yields a non-empty feasible set, we extract it from a pre-trained SNN. 
}

We next show that \eqref{eq:Lip-NCNN-simplified} can be reformulated to a convex SDP by restricting the activation patterns of the hidden-units as well as the signs of the outer layer weights. 
\subsection{Convex Reformulation for Lipschitz Regularization}
We now present our main results. Let $\mathcal{K}^+ = \{k_i^+\}^{\tilde{M}^+}_{i=1}$ and $\ \mathcal{K}^-=\{k_i^-\}^{\tilde{M}^-}_{i=1}$ be two activation pattern multisets on the dataset $\mathbf{X}$ of sizes $\tilde{M}^+ $ and $\tilde{M}^-$, respectively, and let $\{c_i\}_{i=1}^m>0$. Consider the Lipchitz-regularized convex training program (\ref{eq:Lip-convex-FNN}) defined using $\mathcal{K}^+,\ \mathcal{K}^-$, and $\{c_i\}_{i=1}^{\tilde M^++\tilde M^-}>0$, where we write $c_i^+=c_i$ and $c_i^-=c_{i+\tilde M^+}$:
\begin{subequations}\label{eq:Lip-convex-FNN}
\begin{align}
&\min_{\substack{\{\mathbf{w}_{i}^{ +}\}_{i=1}^{\tilde{M}^+}\in\mathbb{R}^d, \\ \{\mathbf{w}_{i}^{-}\}_{i=1}^{\tilde{M}^-}\in\mathbb{R}^d, \\ \rho' \in \mathbb{R}_{\geq0}}}
\begin{aligned}[t]
    &\tfrac{1}{2}\Big\|
     \sum_{i=1}^{\tilde{M}^+}\mathbf{D}(\mathbf{s}^{k_i^+})\mathbf{X}
    \mathbf{w}_i^{+}c^+_i- \\[-0.1em] & \sum_{i=1}^{\tilde{M}^-}\mathbf{D}(\mathbf{s}^{k_i^-})\mathbf{X}
    \mathbf{w}_i^{-}c^-_i-\mathbf{y}\Big\|_2^2
    + \beta_2\rho' \label{eq:Lip-convex-FNN_obj}
\end{aligned}
\\[-0.5em]
\text{s.t.} \quad &  (2\mathbf{D}(\mathbf{s}^{k_i^+})-\mathbf{I})\mathbf{X}\mathbf{w}_i^+\geq0, 
\quad \forall i \in  \llbracket 1, \tilde{M}^+ \rrbracket, 
   \label{eq:Lip-convex-FNN_proj}
\\
    & (2\mathbf{D}(\mathbf{s}^{k_i^-})-\mathbf{I})\mathbf{X}\mathbf{w}_i^-\geq0, 
\quad \forall i \in  \llbracket 1, \tilde{M}^- \rrbracket, 
\\
& {H}_{\mathbf{T}}(\{\mathbf{w}_1^{+},\ldots,\mathbf{w}_{\tilde{M}^-}^{-}\},
   \{c_1^{+},\ldots,-c_{\tilde{M}^-}^{-}\},\rho') \preceq 0, 
   \label{eq:Lip-convex-FNN_H}
\end{align}
\end{subequations}
We now show that \eqref{eq:Lip-NCNN-simplified} and \eqref{eq:Lip-convex-FNN} are equivalent and that we can recover an optimal SNN, i.e., optimal weights, for \eqref{eq:Lip-NCNN-simplified} from an optimal solution of \eqref{eq:Lip-convex-FNN}.

\begin{thm} \label{convex reformulation theorem} Consider the convex problem \eqref{eq:Lip-convex-FNN} constructed using pattern multisets $\mathcal{K}^\pm$ and outer weight absolute values $\{c_j\}_{j=1}^m$ which are also used for \eqref{eq:Lip-NCNN-simplified}. Suppose there exists minimizers $\big(\{\mathbf{u}_j^{\star}\}^{m}_{j=1},\{b_j^{\star}\}^{m}_{j=1},\rho^\star\big)$ and $\big(\{\mathbf{w}_i^{{\pm}\star}\}_{i=1}^{\tilde{M}^\pm},\rho'^\star\big)$ of \eqref{eq:Lip-NCNN-simplified} and \eqref{eq:Lip-convex-FNN}, respectively, which satisfy the following assumptions:
\begin{assum}
The number of hidden units $m$ is greater than the number of nonzero $\mathbf{w}^{\pm\star}_i$ vectors, i.e.,
$m \ge \tilde m=q^++q^-$, where $q^+$ and $q^-$ are the nonzero $\mathbf{w}_i^{+\star}$ and $\mathbf{w}_i^{-\star}$ vectors, respectively. Suppose that the indices are arranged so that non-zero entries $\mathbf{w}_i^{+\star}$ come first, followed by the nonzero entries $\mathbf{w}_i^{-\star}$ and then the zero entries.\label{thm2a1}
\end{assum}
\begin{assum}
 Let $\mathcal{L}^\pm$ be the multiset of activation patterns of $\{\mathbf{u}_j^\star\}_{j=1}^m$ such that $b_j=\pm1$ (positive or negative outer weights). The activation pattern multisets are such that $\mathcal{L}^\pm \subseteq\mathcal{K}^\pm$.\label{thm2a2}
\end{assum}
Then, the optimal value $p_{\eqref{eq:Lip-convex-FNN}}^\star$ for (\ref{eq:Lip-convex-FNN}) is equal to the optimal value $p_{\eqref{eq:Lip-NCNN-simplified}}^\star$ for (\ref{eq:Lip-NCNN-simplified}). 
\end{thm}
\begin{pf}
We first prove $p_{\eqref{eq:Lip-convex-FNN}}^\star\geq p_{\eqref{eq:Lip-NCNN-simplified}}^\star$ by constructing a feasible solution for (\ref{eq:Lip-NCNN-simplified}) from the optimal solution for (\ref{eq:Lip-convex-FNN}) and show that the objective value of both solutions are equal in each respective problems. Consider the following mapping 
\begin{equation}
\label{eq:mapping_1}
\begin{gathered}
\begin{aligned}
\rho \leftarrow \rho'^\star, \quad&
(\mathbf{u}_j^{},\ b_j^{},)\leftarrow \big
({\mathbf{w}_{j}^{+\star}}{},\  1\big),\ j\in \llbracket 1, {q}^+ \rrbracket
\\
&(\mathbf{u}_j^{},\ b_j^{})\leftarrow \big
({\mathbf{w}_{j}^{{-}\star}}{},\ -1\big),\ j\in \llbracket {q}^+, \tilde m \rrbracket,
\end{aligned}
\end{gathered}
\end{equation}
yields a feasible point for (\ref{eq:Lip-NCNN-simplified}) because the matrix in the LMIs (\ref{eq:Lip-convex-FNN_H}) and (\ref{constraints:LMI_SAPN}) are identical. Assumption \ref{thm2a1} ensures $m$ is large enough to house every nonzero $\mathbf{w}^{\pm \star}_j$ mapped by \eqref{eq:mapping_1}. Plugging the solution to (\ref{eq:obj_split_SAPN}) using the mapping \eqref{eq:mapping_1} yields the same objective value as (\ref{eq:Lip-convex-FNN_obj}). To see this, consider the output for a single hidden unit obtained using the mapping \eqref{eq:mapping_1} on $\mathbf{w}_j^{+\star}$. Because $\mathbf{u}_j$ has the same activation pattern as $\mathbf{w}_j^{+\star}$, its output is 
 $
c_i^+(\mathbf{X}\mathbf{u}_i)_+ =c_i^+\mathbf{D}(\mathbf{s}^{k_i^+})\mathbf{X}
\mathbf{u}_i^{+} = c_i^+\mathbf{D}(\mathbf{s}^{k_i^+})\mathbf{X}
\mathbf{w}_i^{+\star}.
 $
Similar arguments apply to $\mathbf{w}_i^-$. Hence, the squared error term remains the same and so does the objective. This proves $p_{\eqref{eq:Lip-NCNN-simplified}}^\star\leq p_{\eqref{eq:Lip-convex-FNN}}^\star$. 

We now prove the converse inequality by constructing a feasible solution to (\ref{eq:Lip-convex-FNN}). 
Consider the mapping:
\begin{equation}
\label{eq:construction_2}
\begin{gathered}
\begin{aligned}
\rho'\leftarrow\rho^\star,\quad
\mathbf{w}_{h(j)}^{+}&\leftarrow&\mathbf{u}_j^{\star}&\quad \forall j\ \ \text{s.t.}\ \ b^\star_j=1\\
\mathbf{w}_{h(j)}^{-}&\leftarrow&\mathbf{u}_j^{\star} &\quad \forall j\ \ \text{s.t.}\ \ b^\star_j=-1,
\end{aligned}
\end{gathered}
\end{equation}
where $h(j)$ is defined so that the $\mathbf{w}_{h(j)}^\pm$ corresponding to $\mathbf{u}_j^\star$ is such that its pattern $\mathbf{s}^{k_{h(j)}}$ matches the pattern of $\mathbf{u}_j^\star$. Assumption \ref{thm1a2} ensures that $h(j)$ is defined for $j\in\llbracket 1, M \rrbracket$. As for the forward inequality, the matrix in the LMIs \eqref{eq:Lip-convex-FNN_H} and \eqref{constraints:LMI_SAPN} are identical, and hence are satisfied. The conic constraints \eqref{eq:Lip-convex-FNN_proj} are also satisfied because $h(j)$ maps $\mathbf{u}_j^\star$ to the $\mathbf{w}_{h(j)}^{\pm}$ with the corresponding conic constraint. 
Hence, the constructed solution is feasible for \eqref{eq:Lip-convex-FNN}, leading to $p_{\eqref{eq:Lip-NCNN-simplified}}^\star\geq p_{\eqref{eq:Lip-convex-FNN}}^\star$. With both sides of the inequalities proven to hold, we have $p_{\eqref{eq:Lip-convex-FNN}}^\star= p_{\eqref{eq:Lip-NCNN-simplified}}^\star$. 
\qed
\end{pf}
Assumptions \ref{thm2a2} of Theorem~\ref{convex reformulation theorem}, is, however, impractical as one needs the optimal activation patterns. Instead of solving \eqref{eq:Lip-NCNN-simplified} to optimality, we propose to use the pattern multiset of a pre-trained SNN to obtain a pattern-based convex restriction of \eqref{eq:Lip-NCNN-simplified} and thus of the Lipschitz-regularized~\eqref{eq:Lip-NCNN}.
\subsection{Pattern-based Convex Restriction} \label{sec: pattern based restriction}
Assumptions~\ref{thm1a2} and~\ref{thm2a2} in both Corollary~\ref{equivalence convex l2} and Theorem~\ref{convex reformulation theorem}, respectively, are not practical because it requires the patterns of an optimal solution beforehand. Alternatively, our approach considers the $\ell_2$- or Lipschitz-regularized convex programs~\eqref{eq:Lip-convex-NN-pilanci} or \eqref{eq:Lip-convex-FNN} to solve the $\ell_2$- or Lipschitz-regularized~\eqref{eq:Lip-NCNN}, respectively, using a possibly sub-optimal $\mathcal{K^\pm}$ from an initial solution, e.g., an SNN trained with another possibly non-convex approach like SGD. The resulting $\ell_2$- or Lipschitz-regularized convex program is shown to be a convex restriction of the $\ell_2$- or Lipschitz-regularized~\eqref{eq:Lip-NCNN}, respectively, which contains the initial solution. 
\begin{cor} \label{restriction corollary}
Consider a feasible point to two problems: the $\ell_2$- and Lipschitz-regularized~\eqref{eq:Lip-NCNN} with objective values $p_{\ell_2\text{-}\eqref{eq:Lip-NCNN}}$ and $p_{\textup{lip-}\eqref{eq:Lip-NCNN}}$, respectively. Let $\{\mathbf{u}_j\}_{j=1}^m$ and $\{\alpha_j\}_{j=1}^m$ be the SNN weights of this feasible point. Let $\mathcal{K}^+=\{k_i^+\}_{i=1}^{\tilde M^+}$ and $\mathcal{K}^-=\{k_i^-\}_{i=1}^{\tilde M^-}$ be the pattern multisets corresponding to hidden-units $\{\mathbf{u}_j\}_{j=1}^m$, where $\tilde M^++\tilde M^-=m$. If the pattern multisets are sub-optimal, then~\eqref{eq:Lip-convex-NN-pilanci} or~\eqref{eq:Lip-convex-FNN} constructed from the multisets is equivalent to $\ell_2$- or Lipschitz-regularized~\eqref{eq:Lip-NCNN} when the following constraints are further imposed:
\begin{subequations} \label{constraints:suboptimal patterns}
    \begin{align}
        (2\mathbf{D}(\mathbf{s}^{k_i^+})-\mathbf{I})\mathbf{X}\mathbf{u}_i&\geq0,\quad i\in\llbracket 1, \tilde{M}^+ \rrbracket 
        \\
        (2\mathbf{D}(\mathbf{s}^{k_i^+})-\mathbf{I})\mathbf{X}\mathbf{u}_i&\geq0,\quad i \in \llbracket M^+,\, m \rrbracket
        \\
        b_i&= \begin{cases}1, &\quad i\in\llbracket 1, \tilde{M}^+ \rrbracket\\ -1,&\quad i \in \llbracket M^+,\, m \rrbracket.\end{cases}
    \end{align}
\end{subequations}
Moreover, ~\eqref{eq:Lip-convex-NN-pilanci} or~\eqref{eq:Lip-convex-FNN} have optimal objective values lower bounded by $p_{\ell_2\text{-}\eqref{eq:Lip-NCNN}}$ or $p_{\textup{lip-}\eqref{eq:Lip-NCNN}}$, respectively.
\end{cor}
\begin{pf}
The proof is the same as Corollary~\ref{equivalence convex l2} and Theorem~\ref{convex reformulation theorem}. In both of them, a mapping $h(j)$ is used to map variables $\{\mathbf{u}_j^\star\}_{j=1}$ from the non-convex problem to the $\mathbf{w}_{h(j)}^\pm$ variables of the convex problem which have the same activation pattern as $\{\mathbf{u}_j^\star\}_{j=1}$. Assumption~\ref{thm1a2} of Corollary~\ref{equivalence convex l2} and Assumption~\ref{thm2a2} of Theorem~\ref{convex reformulation theorem} ensure that $h(j)$ is defined for $j\in\llbracket 1, m \rrbracket$. Removing these assumptions and instead imposing the extra constraints~\eqref{constraints:suboptimal patterns} on the non-convex problem, the mapping remains valid and therefore the proofs for Corollary~\ref{equivalence convex l2} and Theorem~\ref{convex reformulation theorem} remain the same. This shows that \eqref{eq:Lip-convex-FNN} is a convex restriction of \eqref{eq:Lip-NCNN-simplified} which is itself a restriction of the Lipschitz-regularized~\eqref{eq:Lip-NCNN} by Lemma~\ref{lemma:NC_restriction}. Using $h(j)$, the feasible point for the $\ell_2$- or Lipschitz-regularized \eqref{eq:Lip-NCNN} can be mapped to feasible solutions with the same objective values in the convex restrictions~\eqref{eq:Lip-convex-NN-pilanci} or~\eqref{eq:Lip-convex-FNN}, respectively. Therefore,~\eqref{eq:Lip-convex-NN-pilanci} or~\eqref{eq:Lip-convex-FNN} have an optimal objective value lower bounded by $p_{\ell_2\text{-}\eqref{eq:Lip-NCNN}}$ or~$p_{\textup{lip-}\eqref{eq:Lip-NCNN}}$, respectively.~\qed
\end{pf}
Corollary~\ref{restriction corollary}'s restriction forces the SNNs to have the activation patterns $\mathcal{K}^\pm$, which might be sub-optimal. In the special case where $\mathcal{K}^\pm$ are optimal, \eqref{constraints:suboptimal patterns} can be omitted without impacting the optimal objective value because $\mathcal{K}^\pm$ now satisfy the Assumption~\ref{thm1a2} of Corollary~\ref{equivalence convex l2} and Assumption~\ref{thm2a2} of Theorem~\ref{convex reformulation theorem}.

Lastly, we consolidate our analysis and propose Algorithm~\ref{alg:training_lipschitz} 
as a post-processing method to solve the $\ell_2$- and Lipschitz-regularized~\eqref{eq:Lip-NCNN}. Specifically, our approach takes a pre-trained SNN and further reduces its training loss using our pattern-based restriction for both problems.
\begin{algorithm}[tb]
\caption{{Pattern-based convex restriction}}
\label{alg:training_lipschitz}
\begin{algorithmic}[1]
\Require SNN parameters $\{\mathbf{u}_j,\alpha_j\}_{j=1}^{m}$, number of AM iterations $iters$, regularization flag $reg\in\{\ell_2,\mathrm{Lip}\}$

\State $\{\mathbf{s}^{k_i^\pm}\}_{i=1}^{\tilde M^\pm} \gets \texttt{extract patterns from }\{\mathbf{u}_j\}_{j=1}^{m}$  \label{line:extract-patterns}
\If{$reg=\mathrm{Lip}$}
\State $\{{c_j}\}_{i=1}^{m},\{(\mathbf{u}'_j)_0\}_{j=1}^m   \gets\{{|\alpha_j|}\}_{i=1}^{m},\{\mathbf{u}_j\}_{j=1}^m$
\State $\begin{aligned}[t]
        (\mathbf{T}_0,\rho_0) \gets &\texttt{solve }\eqref{eq:fazlyab_LMI}\\[-1.0ex]&\text{with } \{(\mathbf{u}'_j)_0\}_{j=1}^m,\{\alpha_j\}_{j=1}^{m}
        \end{aligned}$
\State $v_0 \gets  \tfrac{1}{2}\texttt{loss}\big(\{(\mathbf{u}'_j)_0,\alpha_j\}_{j=1}^m\big)+\beta_2 \rho_0$ 
    \For{$t \gets 1$ \textbf{to} $iters$}
        \State $\begin{aligned}[t]
        (\mathbf{T}_t,\rho_t) \gets &\texttt{solve }\eqref{eq:fazlyab_LMI}\\[-1.0ex]&\text{with } (\{\mathbf{u}'_j\}_{j=1}^m)_{t-1},\{\alpha_j\}_{j=1}^{m}
        \end{aligned}$ \label{line:solve-for-T}

        \State $\begin{aligned}[t]
        (\{\mathbf{u}'_j\}_{j=1}^{m})_t\  \gets &\texttt{solve }\eqref{eq:Lip-convex-FNN}\\[-1.0ex]& \text{with } \mathbf{T}_t,\ \{\mathbf{s}^{k_i^\pm}\}_{i=1}^{\tilde M^\pm},\ \{c_j\}_{j=1}^{m}
        \end{aligned}$\label{algline:solve_lipschitz_train} 
        \State $v_t \gets  \tfrac{1}{2}\texttt{loss}\big(\{(\mathbf{u}'_j)_t,\alpha_j\}_{j=1}^m\big)+\beta_2 \rho_t$ 
    \EndFor
\ElsIf{$reg=\ell_2$}
    \State $\{\mathbf{u}'_j\}_{j=1}^{m},\{\alpha_j'\}_{j=1}^{m} \gets \texttt{solve } (\ref{eq:Lip-convex-NN-pilanci})\ \text{with}\ \{\mathbf{s}^{k_i^\pm}\}_{i=1}^{\tilde M^\pm}$ \label{line:solve-l2-convex}
\EndIf

\State \Return $\{\mathbf{u}'_j,\alpha_j'\}_{j=1}^{m}$
\end{algorithmic}
\end{algorithm}
Consider the input pre-trained SNN. We first order layer weights such that positive weights come before negative ones. We then extract its pattern multisets  $\mathcal{K}^\pm=\{\mathbf{s}^{k_i^\pm}\}_{i=1}^{\tilde M^\pm}$ (Line \ref{line:extract-patterns}). 
For $\ell_2$-regularized-SNNs
, we solve \eqref{eq:Lip-convex-NN-pilanci} using the extracted pattern multiset and obtain the SNN via the mapping \eqref{eq:mapping_1} (Line \ref{line:solve-l2-convex}). For Lipschitz-regularized SNNs, 
we enter the loop and alternate between solving \eqref{eq:fazlyab_LMI} to obtain $\mathbf{T}$ (Line \ref{line:solve-for-T}) and solving the pattern-based restriction \eqref{eq:Lip-convex-FNN} using the extracted pattern multisets and outer weight absolute values $\{c_j\}_{j=1}^{m}=\{|\alpha_j|\}_{j=1}^m$, as well as the computed~$\mathbf{T}$ to obtain candidates weights $\{\mathbf{u}'_j\}_{i=1}^{m}$ (Line \ref{algline:solve_lipschitz_train}). We refer to this method as the alternating minimization (AM). Lemma~\ref{lemma:bounded_alpha} establishes that when $\mathbf{T}$ is fixed, some $\mathbf{c}$ yields infeasible~\eqref{eq:Lip-NCNN-simplified}. We ensure the convex restriction \eqref{eq:Lip-convex-FNN} is always feasible by taking $\mathbf{c}$ from a feasible point of Lipschitz-regularized \eqref{eq:Lip-NCNN} so that Corollary~\ref{restriction corollary} holds. We note that a similar AM is used in \cite{Lipschitz_Pauli} where $\mathbf{T}$ is updated between iterations of their ADMM algorithm. Next, we demonstrate that AM decreases the objective value at each iteration because it solves a succession of convex restrictions to optimality. 
{
\begin{cor}
Consider the Lipschitz-regularized~\eqref{eq:Lip-NCNN} and let the diagonal matrix $\mathbf{T} \in \mathbb{R}^{m \times m}_{\geq0}$ be a free variable. The alternating minimization loop of Algorithm~\ref{alg:training_lipschitz} leads to a monotonically decreasing sequence of objective values for the Lipschitz-regularized problem \eqref{eq:Lip-NCNN} which convergences to a minimum that is no worst than the initial one.  
\end{cor}
\begin{pf}
Let $t\in\mathbb{N}$ and consider the weights $\{(\mathbf{u}_j')_{t-1}\}_{j=1}^m,\{\alpha_j\}_{j=1}^m$ and matrix $\mathbf{T}_{t-1}$ from the previous iteration. Iteration $t$'s SNN has a squared error term $e_{t-1}=\texttt{loss}\big(\{(\mathbf{u}'_j)_{t-1},\alpha_j\}_{j=1}^m\big)$ and Lipschitz upper bound $\rho_{t-1}$ yielding the training objective value of $v_{t-1}=\tfrac{1}{2}e_{t-1}+\beta_2 \rho_{t-1}$. Following Line 7, we obtain $\mathbf{T}_t$ and $\rho'_{t-1}$, the latter of which is $\rho'_{t-1}\leq\rho_{t-1}$ and, therefore, $v_{t-1}'=\tfrac{1}{2}e_{t-1}+\beta_2\rho_{t-1}'\leq v_{t-1}$. On Line 8, solving~\eqref{eq:Lip-convex-FNN} with a feasible set containing $\{(\mathbf{u}_j)_{t-1}'\}_{j=1}^m$ and the resulting $\rho_{t-1}'\mathbf{T}_t$, we get $v_t\leq v'_{t-1}\leq v_{t-1}$ because \eqref{eq:Lip-convex-FNN} is solved to optimality. Hence, the sequence $(v_t)$ is monotonically decreasing. 

Next we established the convergence of $(v_t)$. Because~\eqref{eq:Lip-convex-FNN} is a restriction of~the Lipschitz-regularized~\eqref{eq:Lip-NCNN}, we have $v_t \geq p^\star_{\eqref{eq:Lip-NCNN}}$ for all $t$, where is $p^\star_{\eqref{eq:Lip-NCNN}}$ the global minimum of the Lipschitz-regularized~\eqref{eq:Lip-NCNN} achieved if, e.g., Theorem~\ref{convex reformulation theorem}'s assumption are met, and the sequence is bounded below. By the monotone convergence theorem, $(v_t) \to \inf_t v_t \equiv \underline{v}$ and the sequence convergences. Finally, monotonicity also implies that $v_0 \geq \underline{v}$ and the resulting objective value is less than or equal to the initial one.
\qed
\end{pf}
}
In the next section, we illustrate using numerical experiments how Algorithm \ref{alg:training_lipschitz} can be used to improve existing training methods on multiple baseline models from which we extract activation patterns.
\section{Experimental results} \label{sec:experiments}
We compare baseline models with our post-processing method on several UCI regression tasks \cite{uci_ml_repository}. We also synthetically generate regression tasks using the $\texttt{rastrigin}$ \cite{rastrigin1974systems} function. We consider SNNs with 100 hidden units for all of our experiments. In the first experiment, we train models in 10 independent trials using $\ell_2$-regularization via SGD and PGD methods and Lipschitz-regularization via \cite{Lipschitz_Pauli}'s ADMM. We extract the pattern multiset of each trained model and employ Algorithm \ref{alg:training_lipschitz} to obtain convex solutions for
Lipschitz-regularized problem. The convex solutions are referred to as 
$(\cdot)$-C-lip
 where the prefix $(\cdot)$ denotes the method used to obtain the initial solution and its corresponding multiset, viz., SGD, PGD and ADMM.
 For example, (SGD)-C-Lip refers to a convex solution to the Lipschitz-regularized problem using the pattern multisets of a base model trained with SGD. As an additional baseline, we include $(\cdot)$-ADMM which is obtained by taking the base model and running the ADMM algorithm to solve the Lipschitz-regularized~\eqref{eq:Lip-NCNN} initialized with the base model as the initial point. Post training, we evaluate the adversarial MSE on the test set using PGD attacks \cite{PGD}, which would be representative of model performance when deployed in an adversarial setting. 
The adversarial MSE is obtained using attacks described by $\epsilon=0\ (\text{clean}),\ 0.1$ and $\ 0.2$. Because there is a trade-off between minimizing the clean MSE ($\epsilon=0$) and the adversarial MSE ($\epsilon>0$), we are interested in models that { perform most consistently across this trade-off, in the sense that they achieve the lowest median MSE across all $\epsilon$ values}.

In Figure \ref{fig:rel-diff-test-mse}, we observe that $(\cdot)$-C-lip is the only model that achieves the lowest median MSE over some datasets and base models, namely, $\texttt{solar}$ with base models PGD and ADMM, $\texttt{machine}$ with all base models, $\texttt{servo}$ with SGD, and $\texttt{wine}$ with SGD and ADMM.
This illustrates that our convex restriction generally finds better solutions than simply fine-tuning a pre-trained model using Lipschitz-regularization with ADMM. 

In our second experiment, we compare objective values of the baseline models for the $\ell_2$- and Lipschitz-regularized~\eqref{eq:Lip-NCNN} problems against our convex solutions. This experiment serves to confirm our claim that our convex restriction obtains a lower objective and to observe the reduction. Using $\beta_1=\beta_2=0.001$, we solve the $\ell_2$- and Lipschitz-regularized problems~\eqref{eq:Lip-NCNN} for two baselines: SGD for the $\ell_2$ case and ADMM for the Lipschitz case and repeat the process over 10 independent trials. Post training, we employ Algorithm~\ref{alg:training_lipschitz} on the baseline models. We try two different methods of extracting pattern multisets, the first being extracting all 100 activation patterns and grouping them into $\mathcal{K}^+$ and $\mathcal{K}^-$ according to the sign of the outer layer weights. { This method is denoted by C-$\ell_2$ and C-lip. The second method, which reflects $\texttt{Alg1}$ described in~\cite{SCNN_Pilanci}, randomly samples 50 among 100 activation patterns of the pre-trained SNN and sets $\mathcal{K^+}=\mathcal{K}^-$.   We denote this method by samp-$\ell_2$ and samp-lip. By extracting only a fraction of the pre-trained SNN, Corollary~\ref{restriction corollary} does not apply and, therefore, Algorithm \ref{alg:training_lipschitz} is not guaranteed to output an improved solution. 
}
We observe in Table~\ref{tab:objectives} that C-$\ell_2$ and C-lip have positive percentages across all data sets on both $\ell_2$- and Lipschitz-regularized~\eqref{eq:Lip-NCNN}, which is consistent with Corollary~\ref{restriction corollary}. In Table~\ref{tab:objectives} (left), we observe that C-$\ell_2$ has significantly lower objective values ($>5\%$) relative to the SGD baseline on \texttt{servo} and \texttt{machine} and is much higher than samp-$\ell_2$. We further note that on $\texttt{rastrigin}$ and $\texttt{wine}$, samp-$\ell_2$ has significantly worse objective value than the SGD baseline. In Table \ref{tab:objectives} (right), we observe that (C)-lip has the lowest objective across all datasets. We observe significant improvement on $\texttt{servo}$ and $\texttt{machine}$. We also observe that samp-lip has worse objective values than ADMM on all datasets. In both $\ell_2$- and Lipschitz-regularized problems, Algorithm~\ref{alg:training_lipschitz} can output a much worse solution than the initial one when extracting only 50 of the activation patterns (i.e., samp-$\ell_2$/lip) compared to extracting all 100 (i.e., C-$\ell_2$/lip). We observe that the improvement from using Algorithm~\ref{alg:training_lipschitz} on the objective value is less apparent in the Lipschitz-regularized~\eqref{eq:Lip-NCNN} case when compared against ADMM than in the $\ell_2$-regularized~\eqref{eq:Lip-NCNN} case when compared against SGD.

\begin{table*}[t]
\centering
\caption{Relative differences of objective values with respect to the baseline method across several datasets, each evaluated over 10 random trials. The relative difference is computed as $\frac{baseline-(\cdot)}{baseline}$ where $baseline$ is the objective obtained using SGD and ADMM for the $\ell_2$- and Lipschitz-regularized problems, respectively. \textbf{Left:} $\ell_2$-regularized objectives relative to the SGD baseline. \textbf{Right:} Lipschitz-regularized objectives relative to the ADMM baseline. Larger positive percentages are better (highlighted in red), while negative percentages are worse (highlighted in blue). The shading intensity reflects the magnitude of the percentage differences, with thresholds at $1\%,\ 5\%$, and $10\%$. On each dataset and column (Min, Max, Avg), the best percentage is in bold.}
\label{tab:objectives}
\begin{minipage}[t]{0.48\textwidth}
\centering
\small
\begin{tabular}{llrrr}
\toprule
Dataset & Method & Min & Max & Avg \\
\midrule
\texttt{servo}    
                   & C-$\ell_2$     & \shadered \textbf{+20.76\%}      & \shadered \textbf{+25.42\%}      & \shadered \textbf{+23.71\%}      \\
                   & samp-$\ell_2$  & \shadered +16.10\%               & \shadered +14.42\%               & \shadered +14.00\%               \\
\addlinespace
\texttt{machine}   
                   & C-$\ell_2$     & \shademedred \textbf{+9.07\%}    & \shadered \textbf{+11.08\%}      & \shademedred \textbf{+9.74\%}    \\
                   & samp-$\ell_2$  & \shadelightred +4.89\%           & \shademedred +5.24\%             & \shademedred +5.63\%             \\
\addlinespace
\texttt{solar}   
                   & C-$\ell_2$     & \shadelightred \textbf{+2.84\%}  & \shadelightred \textbf{+4.27\%}  & \shadelightred \textbf{+3.76\%}  \\
                   & samp-$\ell_2$  & \shadelightblue-3.84\%                & \shadelightblue-1.94\%                & \shadelightblue-3.21\%                \\
\addlinespace
\texttt{rastrigin}
                   & C-$\ell_2$     & \shadelightred \textbf{+3.70\%}  & \shadelightred \textbf{+3.93\%}  & \shadelightred \textbf{+3.88\%}  \\
                   & samp-$\ell_2$  & \shademedblue-9.04\%                & \shadeblue-13.48\%               & \shadeblue-13.06\%               \\
\addlinespace
\texttt{wine}     
                   & C-$\ell_2$     & \shadelightred \textbf{+1.19\%}  & \shadelightred \textbf{+2.13\%}  & \shadelightred \textbf{+1.56\%}  \\
                   & samp-$\ell_2$  & \shadeblue-74.60\%               & \shadeblue-82.83\%               & \shadeblue-75.55\%               \\\bottomrule
\end{tabular}
\end{minipage}
\hfill
\begin{minipage}[t]{0.48\textwidth}
\centering
\small
\begin{tabular}{llrrr}
\toprule
Dataset & Method & Min & Max & Avg \\
\midrule

\texttt{servo}     & C-lip    & \shadelightred\textbf{+3.84\%} & \shademedred\textbf{+6.59\%}   & \shademedred\textbf{+5.92\%}   \\
                   & samp-lip & \shadeblue-10.31\%             & \shadeblue-55.02\%             & \shadeblue-27.18\%             \\
\addlinespace
\texttt{machine}   & C-lip    & \shademedred\textbf{+5.85\%}   & \shademedred\textbf{+8.28\%}   & \shademedred\textbf{+7.46\%}   \\
                   & samp-lip & \shadelightblue-2.57\%         & \shadeblue-21.03\%             & \shadeblue-10.32\%             \\
\addlinespace
\texttt{solar}     & C-lip    & \shadelightred\textbf{+2.67\%} & \shademedred\textbf{+5.28\%}   & \shadelightred\textbf{+3.74\%} \\
                   & samp-lip & \shadelightblue-3.60\%         & \shademedblue-5.57\%           & \shadelightblue-3.95\%         \\
\addlinespace
\texttt{rastrigin} & C-lip    & \shadelightred\textbf{+4.40\%} & \shadelightred\textbf{+3.17\%} & \shadelightred\textbf{+3.89\%} \\
                   & samp-lip & \shadeblue-93.43\%             & \shadeblue-125.24\%            & \shadeblue-109.45\%            \\
\addlinespace
\texttt{wine}      & C-lip    & \shadelightred\textbf{+3.50\%} & \shademedred\textbf{+7.20\%}   & \shademedred\textbf{+5.55\%}   \\
                   & samp-lip & \shadeblue-91.15\%             & \shadeblue-124.66\%            & \shadeblue-101.83\%            \\
\addlinespace
\bottomrule
\end{tabular}\end{minipage}
\end{table*}

\begin{figure*}[t]
  \centering
  \begin{subfigure}[b]{0.196\textwidth}
    \centering
    \includegraphics[width=\linewidth]{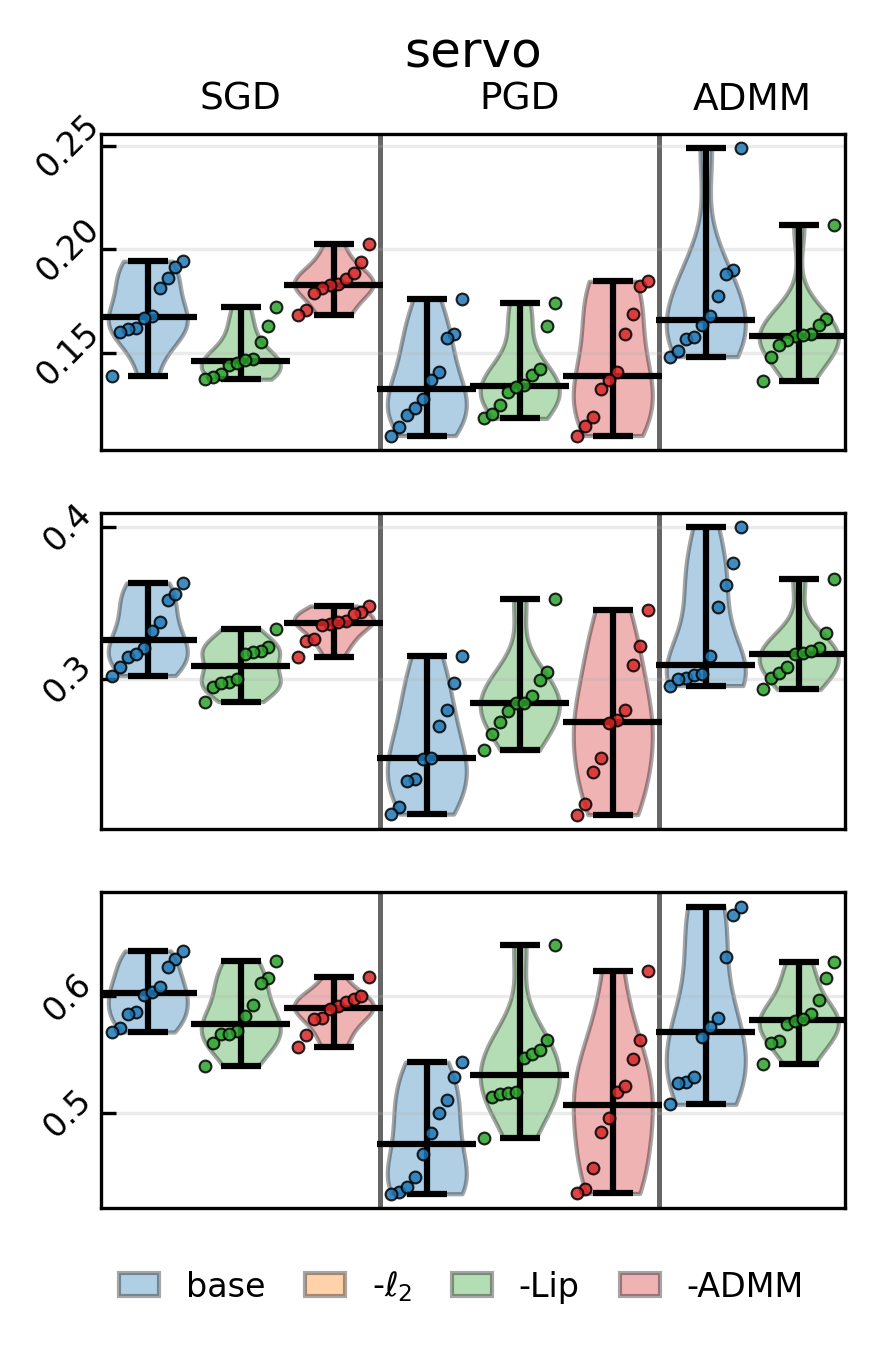}
    \caption{\texttt{servo}}
  \end{subfigure}
  \vspace{0.6em}
  \begin{subfigure}[b]{0.196\textwidth}
    \centering
    \includegraphics[width=\linewidth]{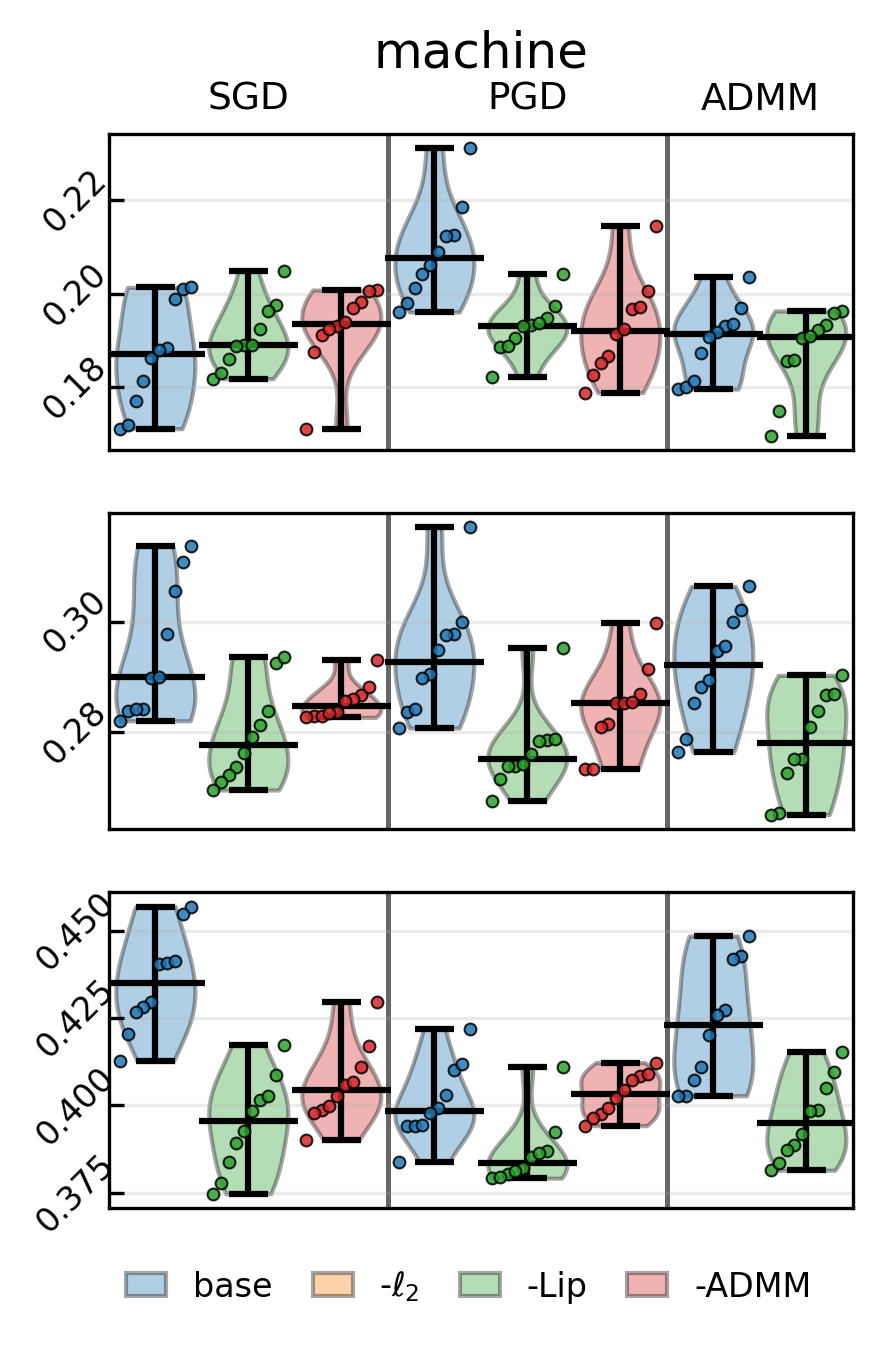}
    \caption{\texttt{machine}}
  \end{subfigure}\hfill
  \begin{subfigure}[b]{0.196\textwidth}
    \centering
    \includegraphics[width=\linewidth]{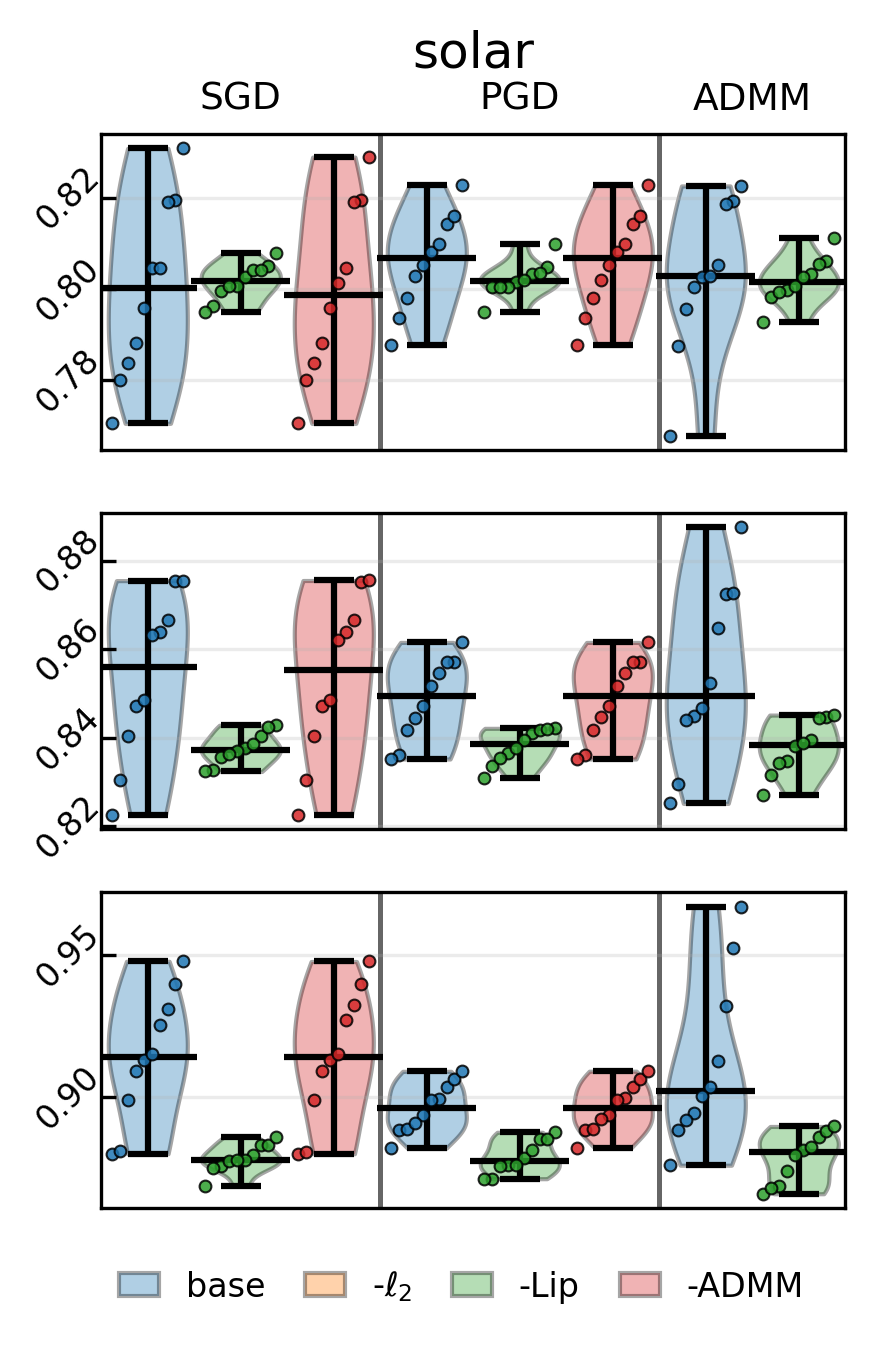}
    \caption{\texttt{solar}}
  \end{subfigure}\hfill
  \begin{subfigure}[b]{0.196\textwidth}
    \centering
    \includegraphics[width=\linewidth]{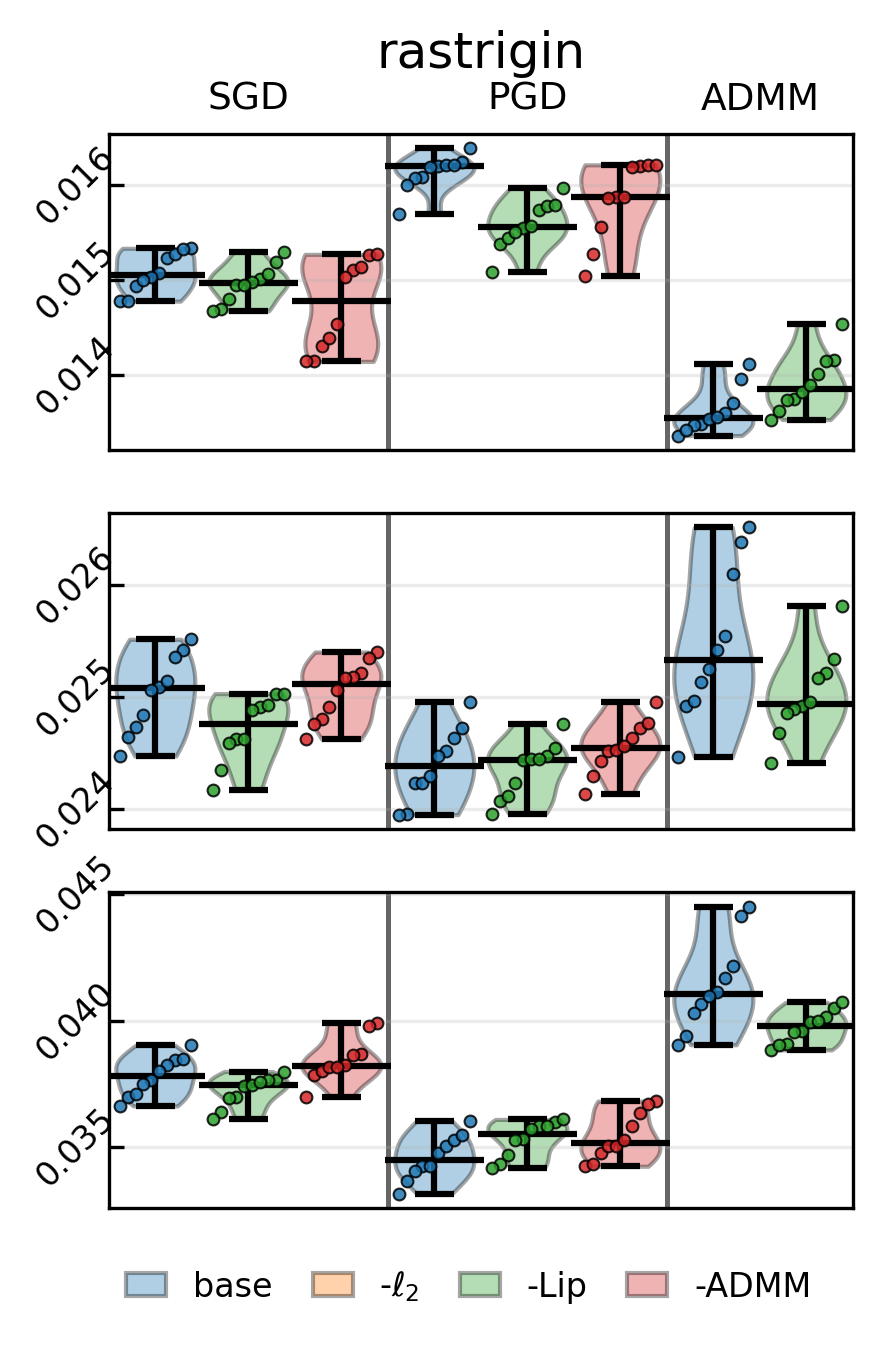}
    \caption{\texttt{rastrigin}}
  \end{subfigure}\hfill
  \begin{subfigure}[b]{0.208\textwidth}
    \centering
    \includegraphics[width=\textwidth]{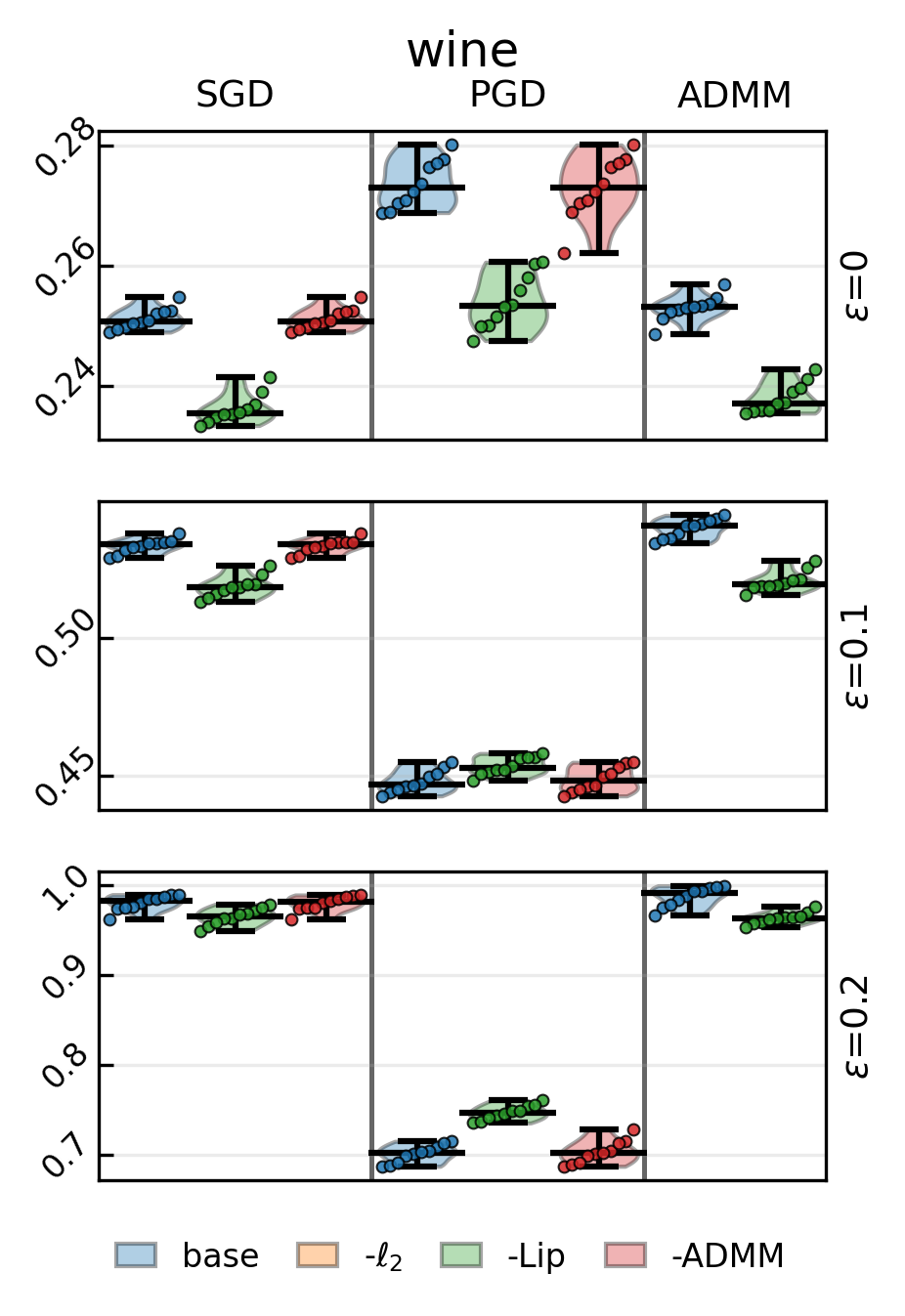}
    \caption{\texttt{wine}}
  \end{subfigure}\hfill

  \caption{Adversarial MSE of base models and convex models on various datasets using PGD attacks of magnitudes $\epsilon\in\{0,\ 0.1,\ 0.2\}$. The columns (SGD, PGD, ADMM) represent the method used to train the base model. Each violin plot shows the distribution over 10 trials, with bars indicating the minimum, median and the maximum.}
  \label{fig:rel-diff-test-mse}
\end{figure*}

\section{Conclusion}
We propose a post-processing method for the Lipschitz-regularized SNN training problem which obtains a solution that is no worse given any feasible solution. We propose a restriction to the Lipschitz-regularized training problem and demonstrated that it admits an equivalent convex program. We design an algorithm which takes any feasible point for the Lipschitz regularization training to obtain a solution that is no worse by solving a convex restriction. We demonstrated experimentally that our post-processing method yields solutions that are both more accurate and robust to adversarial attacks of varying attack sizes. 
In future work, we aim to extend our approach to deep NNs as 
it had been shown that a convex reformulation akin to \eqref{eq:Lip-convex-NN-pilanci} can be obtained for training deep NNs \cite{ergen_DNN}. Additionally, instead of stopping at the solution of the convex restriction, methods to sequentially improve the pattern multiset selection, e.g., continuing to use SGD or ADMM, could be proposed to further move towards the global optimum of { Theorem \ref{convex reformulation theorem}}.

\bibliographystyle{abbrv}       
\bibliography{autosam}         

\appendix
\end{document}